\documentclass{llncs}
\usepackage{amsmath}
\usepackage{amssymb}
\usepackage{times}
\usepackage{indentfirst}
\usepackage{graphicx}
\usepackage{epsfig}
\usepackage{epstopdf}
\usepackage[bookmarks=false]{hyperref}

\begin{document}

\title{Covering matroid}

\author{Yanfang Liu, William Zhu~\thanks{Corresponding author.
E-mail: williamfengzhu@gmail.com (William Zhu)}}
\institute{
Lab of Granular Computing\\
Zhangzhou Normal University, Zhangzhou 363000, China}

%%%%%%%%%%%%%%%%%%%%%%%%%%%%%%%%%%%%%%%%%%%%%%%%%%%%%%%%%%%%%%%%%%%%%%%%%%%%%%%%%%%%%%%%%%%%%%%%%%%%%%%%%%%%%%%%%%%%%%%%%%%%%%
% Enter your name between curly braces

%\date{\today}  % Enter your date or \today between curly braces

\date{\today}          % Enter your date or \today between curly braces
\maketitle

\begin{abstract}
In this paper, we propose a new type of matroids, namely covering matroids, and investigate the connections with the second type of covering-based rough sets and some existing special matroids.
Firstly, as an extension of partitions, coverings are more natural combinatorial objects and can sometimes be more efficient to deal with problems in the real world.
Through extending partitions to coverings, we propose a new type of matroids called covering matroids and prove them to be an extension of partition matroids.
Secondly, since some researchers have successfully applied partition matroids to classical rough sets, we study the relationships between covering matroids and covering-based rough sets which are an extension of classical rough sets.
Thirdly, in matroid theory, there are many special matroids, such as transversal matroids, partition matroids, 2-circuit matroid and partition-circuit matroids.
The relationships among several special matroids and covering matroids are studied.

\end{abstract}

%%%%%%%%%%%%%%%%%%%%%%%%%%%%%%%%%%%%
                                                                                                                                                                                                                          %%%%%%%%%%%%%%%%%%%%%%%%%%%%%%%%%%%%

\section{Introduction}
Matroid theory~\cite{Lai01Matroid,Mao06TheRelation} proposed by Whitney is a generalization of linear algebra and graph theory.
Integrating the characteristics of linear algebra and graph theory, matroids have sound theoretical foundations and wide applications.
In theory, matroids have powerful axiomatic systems which provide a platform for connecting them with other theories, such as rough sets~\cite{LiLiu12Matroidal,LiuZhu12characteristicofpartition-circuitmatroid,LiuZhuZhang12Relationshipbetween,TangSheZhu12matroidal,WangZhuZhuMin12matroidalstructure}, generalized rough sets based on relations~\cite{LiuZhu12Matroidalstructureofroughsets,WangZhuMin11TheVectorially,ZhangWangFengFeng11reductionofrough,ZhuWang11Matroidal,ZhuWang11Rough}, covering-based rough sets~\cite{HuangZhu12geometriclatticestructure,WangZhu11Matroidal,WangZhuMin11Transversal,WangZhuZhuMin12quantitative} and geometric lattices~\cite{AignerDowling71matchingtheory,Matus91abstractfunctional}.
In application, matroids have been successfully applied to diverse fields, such as combinatorial optimization~\cite{Lawler01Combinatorialoptimization}, algorithm design~\cite{Edmonds71Matroids},  information coding~\cite{RouayhebSprintsonGeorghiades10ontheindex} and cryptology~\cite{DoughertyFreilingZeger07Networks}.

Rough set theory proposed by Pawlak~\cite{Pawlak82Rough,Pawlak91Rough} is a tool to deal with inexact, uncertain and insufficient knowledge in information systems.
It has been successfully applied to pattern recognition, expert system, medical diagnosis, environmental science, biology, biochemistry,
chemistry psychology, conflict analysis, economics and process control.
However, rough set theory is based on equivalence relations or partitions which are restrictive for many applications.
To solve this problem, some extensions of rough sets are proposed, such as generalized rough sets base on relations~\cite{Kryszkiewicz98Rough,LiuWang06Semantic,QinYangPei08generalizedroughsets,SlowinskiVanderpooten00AGeneralized,Yao98Relational,ZhuWang06Binary}, and covering-based rough sets~\cite{BonikowskiBryniarskiSkardowska98extensionsandintentions,Bryniarski89ACalculus,TsangChengLeeYeung04ontheupper,XuZhang07measuring,Zhu07Topological,ZhuWang07OnThree,ZhuWang12thefourthtype}.

In order to broaden the theoretical and application areas of rough sets and matroids, some researchers have combined them with each other.
For example, Wang et al.~\cite{WangZhuZhuMin12matroidalstructure} have successfully applied 2-circuit matroids to attribution reduction in classical rough sets.
In this paper, we propose a new type of matroids, namely covering matroids, and study the connections with the second type of covering-based rough sets and some existing special matroids.

First, similar to partition matroids, we introduce a new type of matroids called covering matroids.
As we know, the concept of coverings is an extension of the concept of partitions, and coverings are more natural combinatorial objects and can sometimes be more efficient to deal with problems than partitions.
Based on classical rough sets, we obtain covering-based rough sets through extending partitions to coverings.
Similarly, we want to extend partition matroids~\cite{LiuChen94Matroid} to covering matroids through coverings instead of partitions.
We see partition matroid is induced by a partition of a universe and a group of nonnegative integers.
However, a covering can not always induce a matroid using covering blocks instead of partition ones.
In~\cite{LiuZhuZhang12Relationshipbetween}, we represent a partition matroid as a union of $k$-rank matroids induced by all the partition blocks.
In~\cite{Lai01Matroid}, we see the union of any group of matroids is also a matroid.
Therefore, for a covering of a universe, we obtain a matroid called a covering matroid through the union of $k$-matroids induced by the covering blocks.
Moreover, we prove covering matroids are an extension of partition matroids.

Second, since in~\cite{LiuZhuZhang12Relationshipbetween,WangZhuZhuMin12matroidalstructure} partition matroids are successfully used to investigate classical rough sets, we attempt to apply covering matroids to study covering-based rough sets which are an extension of classical rough sets.
In this paper, we focus on the second type of covering-based rough sets and obtain the following results.
For a covering of a universe and a group of positive integers, a covering block can be expressed by the rank function and the closure operator of the $k$-rank matroid induced by the covering block with respect to the corresponding integer.
The neighborhood of an element of the covering is the intersection of the covering blocks containing the element, then the neighborhood can be represented by $k$-rank matroids induced by the covering blocks.
Moreover, the second type of covering lower and upper approximation operators are also expressed by the $k$-rank matroids induced by all the covering blocks.
For a covering of a universe and a group of positive integers, a $k$-rank matroid induced by a covering block can be expressed by the covering matroid induced by the covering.
Therefore, some concepts of covering-based rough sets can be represented by the covering matroid.

Third, since covering matroids are an extension of partition matroids, it is nature to investigate the relationships among covering matroids and some existing special matroids, such as transversal matroids~\cite{Lai01Matroid}, 2-circuit matroids~\cite{WangZhuZhuMin12matroidalstructure} and partition-circuit matroids~\cite{LiuZhu12characteristicofpartition-circuitmatroid}.
We prove that a matroid is a transversal matroid, then it is a covering matroid.
When a matroid is a partition matroid and a transversal matroid, then it is a 2-circuit matroid.
In~\cite{LiuZhu12characteristicofpartition-circuitmatroid}, we prove that the dual of a 2-circuit matroid is a partition-circuit matroid.
Moreover, we define a new type of matroids, namely double-circuit matroids and prove this new matroid is an identically self-dual matroid.
When a matroid is a partition-circuit matroid and a 2-circuit matroid, it is a double-circuit matroid.

The rest of this paper is organized as follows: In Section~\ref{S:backgrounds}, we recall some basic definitions of covering-based rough sets and matroids.
Section~\ref{S:coveringmatroid} proposes a new type of matroids, namely covering matroids, and proves this type of matroids is an extension of partition matroids.
In Section~\ref{S:relationshipsbetweencovering-basedandcoveringmatroids}, we connect covering matroids with the second type of covering-based rough sets.
Section~\ref{S:relationshipamongcoveringmatroidandothermatroids} studies the relationships between covering matroids and some special matroids.
Finally, we conclude this paper in Section~\ref{S:conclusions}.

%%%%%%%%%%%%%%%%%%%%%%%%%%%%%%%%%%%%%%%%%%%%%%%%%%%%%%%%%%%%%%%%%%%%%%%%%%%%%%%
%%%%%%%%%%%%%%%%%%%%%%%%%%%%%%%%%%%%%%%%%%%%%%%%%%%%%%%%%%%%%%%%%%%%%%%%%%%%%%%
\section{Backgrounds}
\label{S:backgrounds}
In this section, we present some basic definitions and related results of covering-based rough sets and matroids which will be used in this paper.

\subsection{Covering-based rough set model}
Classical rough set theory proposed by Pawlak~\cite{Pawlak82Rough} in 1982 is based on partitions on a universe.
Covering-based rough set theory is obtained through extending a partition to a covering.

\begin{definition}(Covering~\cite{ZhuWang03Reduction})
\label{D:covering}
Let $U$ be a universe of discourse and $\mathbf{C}$ a family of subsets of $U$.
If none of subsets in $\mathbf{C}$ is empty and $\cup\mathbf{C}=U$, then $\mathbf{C}$ is called a covering of $U$.
\end{definition}

It is clear that a partition of $U$ is certainly a covering of $U$, so the concept of a covering is an extension of the concept of a partition.

Throughout this paper, the universe $U$ is non-empty finite unless otherwise stated.

In the following, we will present some definitions of covering-based rough sets used in this paper.

\begin{definition}(Covering approximation space~\cite{BonikowskiBryniarskiSkardowska98extensionsandintentions})
Let $U$ be a universe and $\mathbf{C}$ a covering of $U$.
We call the ordered pair $(U, \mathbf{C})$ a covering approximation space.
\end{definition}

Neighborhood is an important concept in covering-based rough sets and has been widely applied to knowledge classification and feature selection.

\begin{definition}(Neighborhood~\cite{Zhu09RelationshipBetween})
\label{D:neighborhood}
Let $(U, \mathbf{C})$ be a covering approximation space.
For all $x\in U$, $N_{\mathbf{C}}(x)=\cap\{K\in\mathbf{C}:x\in K\}$ is called the neighborhood of $x$ with respect to $\mathbf{C}$.
When there is no confusion, we omit the subscript $\mathbf{C}$.
\end{definition}

Zakowski first extended classical rough sets to covering-based rough sets by using a covering instead of a partition~\cite{Zakowski83Approximations}.
In order to satisfy different needs of application areas, many types of covering-based rough sets are established.
In this paper, we investigate only the second type of covering-based rough sets proposed by Pomykala~\cite{Pomykala87Approximation}.

\begin{definition}(The second type of covering lower and upper approximations~\cite{Pomykala87Approximation})
\label{D:coveringapproximation}
Let $(U, \mathbf{C})$ be a covering approximation space.
For any $X\subseteq U$,\\
\centerline{$SL_{\mathbf{C}}(X)=\cup\{K\in\mathbf{C}:K\subseteq X\}$,}\\
\centerline{$SH_{\mathbf{C}}(X)=\cup\{K\in\mathbf{C}:K\cap X\neq\emptyset\}$,}
where $SL, SH$ are the second type of covering lower, upper approximation operators, respectively.
When there is no confusion, we omit the subscript $\mathbf{C}$.
\end{definition}

\subsection{Matroid model}
Matroids have many different but equivalent definitions.
In the following definition, we will introduce one from the viewpoint of independent sets.

\begin{definition}(Matroid~\cite{Lai01Matroid})
\label{D:matroid}
A matroid is a pair $M=(U, \mathbf{I})$ consisting of a finite universe $U$ and a collection $\mathbf{I}$ of subsets of $U$ called independent sets satisfying the following three properties:\\
(I1) $\emptyset\in\mathbf{I}$;\\
(I2) If $I\in \mathbf{I}$ and $I'\subseteq I$, then $I'\in \mathbf{I}$;\\
(I3) If $I_{1}, I_{2}\in \mathbf{I}$ and $|I_{1}|<|I_{2}|$, then there exists $u\in I_{2}-I_{1}$ such that $I_{1}\cup \{u\}\in \mathbf{I}$, where $|I|$ denotes the cardinality of $I$.
\end{definition}

Since the above definition of matroids focuses on independent sets, it is also called the independent set axioms of matroids.
For a matroid, if a set is not an independent set, then it is called a dependent set of the matroid.
Based on the dependent sets, we introduce the circuits of a matroid.
For this purpose, some denotations are presented.

\begin{definition}
Let $\mathbf{A}$ be a family of subsets of $U$.
One can denote\\
$Opp(\mathbf{A})=\{X\subseteq U:X\notin\mathbf{A}\}$,\\
$Min(\mathbf{A})=\{X\in\mathbf{A}:\forall Y\in\mathbf{A}, Y\subseteq X\Rightarrow X=Y\}$,\\
$Max(\mathbf{A})=\{X\in\mathbf{A}:\forall Y\in\mathbf{A}, X\subseteq Y\Rightarrow X=Y\}$.
\end{definition}

The dependent sets of a matroid generalize the linear dependence in vector spaces and the cycle in graphs.
The circuit of a matroid is a minimal dependent set.

\begin{definition}(Circuit~\cite{Lai01Matroid})
\label{D:circuit}
Let $M=(U, \mathbf{I})$ be a matroid.
Any minimal dependent set in $M$ is called a circuit of $M$, and we denote the family of all circuits of $M$ by $\mathbf{C}(M)$, i.e., $\mathbf{C}(M)=Min(Opp(\mathbf{I}))$.
\end{definition}

Any maximal independent set of a matroid is a base.
The bases of a matroid generalize the maximal linear independence in vector spaces.

\begin{definition}(Base~\cite{Lai01Matroid})
\label{D:base}
Let $M=(U, \mathbf{I})$ be a matroid.
Any maximal independent set in $M$ is called a base of $M$, and the family of all bases of $M$ is denoted by $\mathbf{B}(M)$, i.e., $\mathbf{B}(M)=Max(\mathbf{I})$.
\end{definition}

As a generalization of linear algebra, a matroid has its rank function which generalizes the maximal independence in vector subspaces.

\begin{definition}(Rank function \cite{Lai01Matroid})
\label{D:rank}
Let $M=(U, \mathbf{I})$ be a matroid and $X\subseteq U$.\\
\centerline{$r_{M}(X)=max\{|I|:I\subseteq X, I\in \mathbf{I}\}$,}
where $r_{M}$ is called the rank function of $M$.
\end{definition}

In order to represent the dependency between an element and a subset of a universe, we introduce the closure operator of a matroid.

\begin{definition}(Closure~\cite{Lai01Matroid})
\label{D:closure}
Let $M=(U, \mathbf{I})$ be a matroid and $X\subseteq U$.
For any $u\in U$, if $r_{M}(X)=r_{M}(X\cup\{u\})$, then $u$ depends on $X$.
The subset of all elements depending on $X$ of $U$ is called the closure of $X$ with respect to $M$ and denoted by $cl_{M}(X)$:\\
\centerline{$cl_{M}(X)=\{u\in U:r_{M}(X)=r_{M}(X\cup\{u\})\}$.}
\end{definition}

Duality is one of important characteristics of matroids, which can generate a new matroid through a given matroid.
And the new matroid is called the dual matroid of the given one which generalizes the orthogonal complement of a vector space.
In the following definition, we introduce the dual matroid of a matroid from the viewpoint of bases.

\begin{definition}(Dual matroid~\cite{Lai01Matroid})
\label{D:dualmatroid}
Let $M=(U, \mathbf{I})$ be a matroid and $\mathbf{B}^{*}(M)=\{B^{c}:B\in\mathbf{B}(M)\}$.
Then $\mathbf{B}^{*}(M)$ is the family of bases of a matroid which is called the dual matroid of $M$ and denoted by $M^{*}$.
\end{definition}

In the following definition, we introduce a special type of matroids called partition matroids.

\begin{definition}(Partition matroid~\cite{LiuChen94Matroid})
\label{D:partitionmatroid}
Let $\mathbf{P}=\{P_{1}, \cdots, P_{m}\}$ be a partition of $U$ and $k_{1}, \cdots, k_{m}$ a group of nonnegative integers.
Then $M(\mathbf{P}; k_{1}, \cdots, k_{m})=(U, \mathbf{I}(\mathbf{P}; k_{1},$ $ \cdots, k_{m}))$ is a matroid, where $\mathbf{I}(\mathbf{P}; k_{1}, \cdots$, $ k_{m})=\{X\subseteq U: |X\cap P_{i}|\leq k_{i}, 1\leq i\leq m\}$, and it is called the partition matroid induced by $\mathbf{P}$ with respect to $k_{1}, \cdots, k_{m}$.
\end{definition}

\section{Covering matroid}
\label{S:coveringmatroid}
According to Definition~\ref{D:covering}, we see the concept of a covering is an extension of the concept of a partition.
Similar to partition matroids, can a covering induce a matroid?

\begin{remark}
Let $\mathbf{C}=\{K_{1}, \cdots, K_{m}\}$ be a covering of $U$ and $k_{1}, \cdots, k_{m}$ a group of nonnegative integers.
Then $M(\mathbf{C}; k_{1}, \cdots, k_{m})=(U, \mathbf{I}(\mathbf{C}; k_{1},$ $ \cdots, k_{m}))$ is not always a matroid, where $\mathbf{I}(\mathbf{C}; k_{1}, \cdots$, $ k_{m})=\{X\subseteq U: |X\cap K_{i}|\leq k_{i}, 1\leq i\leq m\}$.
\end{remark}

An example is given to illustrate the above remark.

\begin{example}
\label{E:example1}
Let $U=\{a, b, c\}$ and $\mathbf{C}=\{K_{1}, K_{2}\}$ where $K_{1}=\{a, b\}, K_{2}=\{b, c\}$.
Suppose $k_{1}=k_{2}=1$, by direct computation, $\mathbf{I}(\mathbf{C}; k_{1}, k_{2})=\{\emptyset, \{a\}, \{b\}, \{c\}, \{a, c\}\}$.
According to (I3) of Definition~\ref{D:matroid}, $\{b\}, \{a, c\}\in\mathbf{I}(\mathbf{C}; k_{1}, k_{2})$ and $|\{b\}|<|\{a, c\}|$, for all $x\in \{a, c\}-\{b\}=\{a, c\}, \{b\}\cup\{x\}\notin\mathbf{I}(\mathbf{C}; k_{1}, k_{2})$, i.e., $\{a, b\}\notin\mathbf{I}(\mathbf{C}; k_{1}, k_{2})$ and $\{b, c\}\notin\mathbf{I}(\mathbf{C}; k_{1}, k_{2})$.
Therefore $M(\mathbf{C}; k_{1}, k_{2})=(U, \mathbf{I}(\mathbf{C};$ $ k_{1}, k_{2}))$ is not a matroid.
\end{example}

Given a covering and a group of nonnegative integers, similar to partition matroids, the covering can not always induce a matroid.
It is nature to ask the following question: how can we combine the covering and the group of nonnegative integers so that the covering can induce a matroid?
In order to solve this issue, we first introduce a type of matroids called $k$-rank matroids.

\begin{definition}($k$-rank matroid~\cite{LiuZhuZhang12Relationshipbetween})
\label{D:krankmatroid}
Let $U$ be a universe, $X\subseteq U$ and $k$ a nonnegative integer.
Then $M_{U}(X, k)=(U, \mathbf{I}(X, k))$ is a matroid where $\mathbf{I}(X, k)=\{Y\subseteq X: |Y|\leq k\}$.
We say $M_{U}(X, k)$ is a $k$-rank matroid on $U$ induced by $X$ with respect to $k$. When there is no confusion, we omit the subscript $U$.
\end{definition}

Union of matroids was introduced by Nash-Williams in 1966.
In the following, we will recall the definition of union of matroids on the same universe.

\begin{definition}(Union of matroids~\cite{Lai01Matroid})
\label{D:unionofmatroids}
Let $M_{1}=(U, \mathbf{I}_{1}), \cdots, M_{m}=(U, \mathbf{I}_{m})$ be a group of matroids.
Then $M=(U, \mathbf{I})$ is a matroid where $\mathbf{I}=\{I_{1}\cup\cdots\cup I_{m}:I_{i}\in \mathbf{I}_{i}, 1\leq i\leq m\}$, and it is called the union of $M_{1}, \cdots, M_{m}$ and denoted by $M=\sum\limits_{i=1}^{m}M_{i}$.
\end{definition}

Given a covering and a group of nonnegative integers, according to Definition~\ref{D:krankmatroid}, any covering block of the covering and an integer can induce a $k$-rank matroid.
According to Definition~\ref{D:unionofmatroids}, we can construct a new matroid through combining the $k$-rank matroids generated by all the covering blocks of the covering.
The new matroid is called a covering matroid.

\begin{definition}(Covering matroid)
\label{D:coveringmatroid}
Let $\mathbf{C}=\{K_{1}, \cdots, K_{m}\}$ be a covering of $U$, $k_{1},$ $ \cdots, k_{m}$ a group of nonnegative integers and $M(K_{i}, k_{i})=(U, \mathbf{I}(K_{i}, k_{i}))$ the $k$-rank matroid induced by $K_{i} (1\leq i\leq m)$.
Then $M(\mathbf{C}; k_{1}, \cdots, k_{m})=(U, \mathbf{I}(\mathbf{C}; k_{1},$ $ \cdots, k_{m}))$ is a matroid, where $\mathbf{I}(\mathbf{C}; k_{1}, \cdots, k_{m})=\{I_{1}\cup\cdots\cup I_{m}:I_{i}\in\mathbf{I}(K_{i}, k_{i}), 1\leq i\leq m\}$.
We say $M(\mathbf{C}; k_{1}, \cdots, k_{m})$ is induced by $\mathbf{C}$ with respect to $k_{1}, \cdots, k_{m}$, denote it as $M(\mathbf{C}; k_{1}, \cdots, k_{m})=\sum\limits_{i=1}^{m}M(K_{i}, k_{i})$ and call it a covering matroid.
\end{definition}

In order to further understand the new type of matroids, an example is given in the following.

\begin{example}(Continued from Example~\ref{E:example1})
\label{E:example2}
The $k$-rank matroids induced by $K_{1}, K_{2}$ are $M(K_{1}, k_{1})=(U, \mathbf{I}(K_{1}, k_{1})), M(K_{2}, k_{2})=(U, \mathbf{I}(K_{2}, k_{2})$.
By direct computation, $\mathbf{I}(K_{1}, k_{1})=\{\emptyset, \{a\}, \{b\}\}, \mathbf{I}(K_{2}, k_{2})=\{\emptyset, \{b\}, \{c\}\}$.
Therefore the covering matroid induced by $\mathbf{C}$ with respect to $k_{1}, k_{2}$ is $M(\mathbf{C}; k_{1}, k_{2})=(U, \mathbf{I}(\mathbf{C}; k_{1}, k_{2}))$, where $\mathbf{I}(\mathbf{C}; k_{1}, k_{2})=\{\emptyset, \{a\}, \{b\}, \{c\}, \{a, b\}, \{a, c\}, \{b, c\}\}$.
\end{example}

Since a covering is an extension of a partition on a universe, we ask a questions: whether a covering matroid is a partition matroid when the covering degenerates into a partition?
In order to solve this question, we introduce the following theorem.

\begin{theorem}(\cite{LiuZhuZhang12Relationshipbetween})
Let $\mathbf{P}=\{P_{1}, \cdots, P_{m}\}$ be a partition on $U$, $k_{1}, \cdots, k_{m}$ a group of nonnegative integers and $M(\mathbf{P}; k_{1}, \cdots, k_{m})$ the partition matroid induced by $\mathbf{P}$ with respect to $k_{1}, \cdots, k_{m}$.
Then $M(\mathbf{P};k_{1}, \cdots, k_{m})=\sum\limits_{i=1}^{m}M(P_{i}, k_{i})$.
\end{theorem}

According to the above theorem, we can easily obtain the following proposition.

\begin{proposition}
Let $\mathbf{C}=\{K_{1}, \cdots, K_{m}\}$ be a covering of $U$, $k_{1}, \cdots, k_{m}$ a group of nonnegative integers and $M(\mathbf{C}; k_{1}, \cdots, k_{m})$ the covering matroid induced by $\mathbf{C}$ with respect to $k_{1}, \cdots, k_{m}$.
If $\mathbf{C}$ is a partition of $U$, then $M(\mathbf{C}; k_{1}, \cdots, k_{m})$ is a partition matroid.
\end{proposition}

In fact, some covering matroids can be represented by partition matroids.
An example is given in the following.

\begin{example}(Continued from Example~\ref{E:example2})
Suppose $\mathbf{P}=\{\{a, b, c\}\}$ is a partition of $U$ and $k = 2$, then $\mathbf{I}(\mathbf{P}; k)=\{\emptyset, \{a\}, \{b\}, \{c\}, \{a, b\}, \{a, c\}, \{b, c\}\}$.
Therefore, $\mathbf{I}(\mathbf{C}; k_{1},$ $ k_{2})=\mathbf{I}(\mathbf{P}; k)$.
\end{example}

However, not all covering matroids can be represented by partition matroids.
A counterexample is presented as follows.

\begin{example}
Let $U=\{a, b, c, d\}$ and $\mathbf{C}=\{K_{1}, K_{2}\}$ a covering of $U$, where $K_{1}=\{a, b\}, K_{2}=\{b, c, d\}$.
Suppose $k_{1}=k_{2}=1$, by direct computation, $\mathbf{I}(K_{1}, k_{1})=\{\emptyset, \{a\}, \{b\}\}, \mathbf{I}(K_{2}, k_{2})=\{\emptyset, \{b\}, \{c\}, \{d\}\}$.
However, it can not be represented by any partition matroid.
\end{example}

\section{Relationships between covering-based rough sets and covering matroids}
\label{S:relationshipsbetweencovering-basedandcoveringmatroids}
In Section~\ref{S:coveringmatroid}, we see that the concept of covering matroids is an extension of the concept of partition matroids.
In~\cite{LiuZhuZhang12Relationshipbetween}, we study the relationships between partition matroids and classical rough sets.
Covering-based rough sets are put forward as an extension of classical rough sets, therefore we will investigate the relationships between the second type of covering-based rough sets and covering matroids in this section.
First, for a given covering, any covering block can be represented by the $k$-rank matroid induced by the covering block with its corresponding positive integer.

\begin{proposition}
\label{P:coveringblock}
Let $\mathbf{C}=\{K_{1}, \cdots, K_{m}\}$ be a covering of $U$ and $k_{1}, \cdots, k_{m}$ a group of positive integers.
Then $K_{i}=cl_{M(K_{i}, k_{i})}^{c}(\emptyset)$ for $i=1, \cdots, m$.
\end{proposition}

\begin{proof}
According to Definition~\ref{D:krankmatroid} and $k_{i}>0$ for $i=1, \cdots, m$, then $\forall x\in K_{i}, \{x\}\in\mathbf{I}(K_{i}, k_{i})$ and $\forall I\in\mathbf{I}(K_{i}, k_{i}), I\subseteq K_{i}$.
Therefore, $r_{M(K_{i}, k_{i})}(K_{i}^{c})=0$.
According to Definition~\ref{D:matroid} and Definition~\ref{D:rank}, $r_{M(K_{i}, k_{i})}(\emptyset)=0$.
According to Definition~\ref{D:closure}, $cl_{M(K_{i}, k_{i})}(\emptyset)=K_{i}^{c}$, i.e., $K_{i}=cl_{M(K_{i}, k_{i})}^{c}(\emptyset)$.
\end{proof}

Given a covering of a universe, we can judge whether an element of the universe belongs to a covering block or not through the $k$-rank matroid induced by the covering block with its corresponding positive integer.

\begin{proposition}
\label{P:anelementcotainedinacoveringblock}
Let $\mathbf{C}=\{K_{1}, \cdots, K_{m}\}$ be a covering of $U$ and $k_{1}, \cdots, k_{m}$ a group of positive integers.
For any $x\in U$, the following conditions are equivalent:\\
(1) $x\in K_{i}$,\\
(2) $\{x\}\in\mathbf{I}(K_{i}, k_{i})$,\\
(3) $r_{M(k_{i}, k_{i})}(\{x\})=1$.
\end{proposition}

\begin{proof}
$(1)\Rightarrow (2)$: Since $x\in K_{i}$ and $k_{i}>0$, according to Definition~\ref{D:krankmatroid}, $\{x\}\in\mathbf{I}(K_{i}, k_{i})$.\\
$(2)\Rightarrow (3)$: According to Definition~\ref{D:rank}, if $\{x\}\in\mathbf{I}(K_{i}, k_{i})$, then $r_{M(k_{i}, k_{i})}(\{x\})=1$.\\
$(3)\Rightarrow (1)$: Suppose that $x\notin K_{i}$.
According to Definition~\ref{D:krankmatroid}, we see that for all $I\in\mathbf{I}(K_{i}, k_{i}), I\subseteq K_{i}$.
Hence $\{x\}\notin\mathbf{I}(K_{i}, k_{i})$.
According to Definition~\ref{D:rank}, $r_{M(k_{i}, k_{i})}(\{x\})=0$ which is contradictory with $(3)$.
Therefore $x\in K_{i}$.
\end{proof}

For a covering of a universe, the neighborhood of an element of the universe is the intersection of the covering blocks including the element.
Therefore, we obtain a proposition about neighborhoods.

\begin{proposition}
Let $\mathbf{C}=\{K_{1}, \cdots, K_{m}\}$ be a covering of $U$ and $k_{1}, \cdots, k_{m}$ a group of positive integers.
For any $x\in U$,\\
\centerline{$N(x)=\cap\{cl_{M(K_{i}, k_{i})}^{c}(\emptyset):r_{M(K_{i}, k_{i})}(\{x\})=1, i=1, \cdots, m\}$.}
\end{proposition}

\begin{proof}
According to Definition~\ref{D:neighborhood}, Proposition~\ref{P:coveringblock} and Proposition~\ref{P:anelementcotainedinacoveringblock}, it is straightforward.
\end{proof}

In the following theorem, we represent the second type of covering upper and lower approximations through $k$-rank matroids.

\begin{theorem}
Let $\mathbf{C}=\{K_{1}, \cdots, K_{m}\}$ be a covering of $U$ and $k_{1}, \cdots, k_{m}$ a group of positive integers.
For all $X\subseteq U$,\\
\centerline{$SL(X)=\cup\{cl_{M(K_{i}, k_{i})}^{c}(\emptyset):r_{M(K_{i}, k_{i})}(X)=r_{M(K_{i}, k_{i})}(K_{i})\}$,}\\
\centerline{$SH(X)=\cup\{cl_{M(K_{i}, k_{i})}^{c}(\emptyset):r_{M(K_{i}, k_{i})}(X)>0\}.~~~~~~~~~~~~~~~~~~~~~~~$}
\end{theorem}

\begin{proof}
According to Definition~\ref{D:coveringapproximation}, for all $X\subseteq U$, $SL(X)=\cup\{K\in\mathbf{C}:K\subseteq X\}, SH(X)=\cup\{K\in\mathbf{C}:K\cap X\neq\emptyset\}$.
According to Definition~\ref{D:krankmatroid}, for all $I\in\mathbf{I}(K_{i}, k_{i}), I\subseteq K_{i}$.
According to Definition~\ref{D:rank}, $r_{M(K_{i}, k_{i})}(X)=r_{M(K_{i}, k_{i})}(X\cap K_{i})$.
If $K_{i}\subseteq X$, then $K_{i}\cap X=K_{i}$.
According to Proposition~\ref{P:coveringblock}, $K_{i}=cl_{M(K_{i}, k_{i})}^{c}(\emptyset)$.
Therefore, $SL(X)=\cup\{cl_{M(K_{i}, k_{i})}^{c}(\emptyset):r_{M(K_{i}, k_{i})}(X)=r_{M(K_{i}, k_{i})}(K_{i})\}$.

According to Definition~\ref{D:krankmatroid} and $k_{i}>0$ for $1\leq i\leq m$, $\{x\}\in\mathbf{I}(K_{i}, k_{i})$ for all $x\in K_{i}$.
If $X\cap K_{i}\neq\emptyset$, then there at least exists $x\in X$ such that $x\in K_{i}$.
Therefore, for all $K\subseteq SH(X)$, $r_{M(K, k)}(X)>0$.
Hence $SH(X)=\cup\{cl_{M(K_{i}, k_{i})}^{c}(\emptyset):r_{M(K_{i}, k_{i})}(X)>0\}$.
\end{proof}

In Section~\ref{S:coveringmatroid}, we see that a covering matroid is the union of a family of $k$-rank matroids.
It is nature to ask the following question: can a $k$-rank matroid be represented by the covering matroid?
In the following proposition, we solve this issue.

\begin{proposition}
Let $\mathbf{C}=\{K_{1}, \cdots, K_{m}\}$ be a covering of $U$ and $k_{1}, \cdots, k_{m}$ a group of nonnegative integers.
Then, for $1\leq i\leq m$,\\
\centerline{$M(K_{i}, k_{i})=M(\mathbf{C}; 0, \cdots, 0, k_{i}, 0, \cdots, 0)$.}
\end{proposition}

\begin{proof}
According to Definition~\ref{D:coveringmatroid}, it is straightforward.
\end{proof}

According the above propositions, we can easily obtain the following results.

\begin{corollary}
Let $\mathbf{C}=\{K_{1}, \cdots, K_{m}\}$ be a covering of $U$ and $k_{1}, \cdots, k_{m}$ a group of positive integers.
Then, for $i=1, \cdots, m$,\\
\centerline{$K_{i}=cl_{M(\mathbf{C}; 0, \cdots, 0, k_{i}, 0, \cdots, 0)}^{c}(\emptyset)$.}
\end{corollary}

\begin{corollary}
Let $\mathbf{C}=\{K_{1}, \cdots, K_{m}\}$ be a covering of $U$ and $k_{1}, \cdots, k_{m}$ a group of positive integers.
For any $x\in U$, the following conditions are equivalent: for $i=1, \cdots, m$,\\
(1) $x\in K_{i}$,\\
(2) $\{x\}\in\mathbf{I}(\mathbf{C}; 0, \cdots, 0, k_{i}, 0, \cdots, 0)$,\\
(3) $r_{M(\mathbf{C}; 0, \cdots, 0, k_{i}, 0, \cdots, 0)}(\{x\})=1$.
\end{corollary}

\begin{corollary}
Let $\mathbf{C}=\{K_{1}, \cdots, K_{m}\}$ be a covering of $U$ and $k_{1}, \cdots, k_{m}$ a group of positive integers.
For any $x\in U$,\\
$N(x)=\cap\{cl_{M(\mathbf{C}; 0, \cdots, 0, k_{i}, 0, \cdots, 0)}^{c}(\emptyset):r_{M(\mathbf{C}; 0, \cdots, 0, k_{i}, 0, \cdots, 0)}(\{x\})=1, i=1, \cdots, m\}$.
\end{corollary}

\begin{corollary}
Let $\mathbf{C}=\{K_{1}, \cdots, K_{m}\}$ be a covering of $U$ and $k_{1}, \cdots, k_{m}$ a group of positive integers.
For any $X\subseteq U$,\\
$SL(X)=\cup\{cl_{M(\mathbf{C}; 0, \cdots, 0, k_{i}, 0, \cdots, 0)}^{c}(\emptyset):r_{M(\mathbf{C}; 0, \cdots, 0, k_{i}, 0, \cdots, 0)}(X)=r_{M(\mathbf{C}; 0, \cdots, 0, k_{i}, 0, \cdots, 0)}(K_{i})\}$,\\
$SH(X)=\cup\{cl_{M(\mathbf{C}; 0, \cdots, 0, k_{i}, 0, \cdots, 0)}^{c}(\emptyset):r_{M(\mathbf{C}; 0, \cdots, 0, k_{i}, 0, \cdots, 0)}(X)>0\}$.
\end{corollary}

\section{Relationships among covering matroids and some special matroids}
\label{S:relationshipamongcoveringmatroidandothermatroids}
In Section~\ref{S:coveringmatroid}, we see that a partition matroid is certainly a covering matroid.
In the following, we will investigate relationships among covering matroids and some special matroids, such as transversal matroids~\cite{Lai01Matroid}, partition matroids~\cite{LiuChen94Matroid}, 2-circuit matroids~\cite{WangZhuZhuMin12matroidalstructure}, partition-circuit matroids~\cite{LiuZhu12characteristicofpartition-circuitmatroid} and double-circuit matroids which will be defined in this paper.

\subsection{Covering matroids and transversal matroids}
Transversal theory has a close relationship with graph theory.
As a generalization of graph theory, matroid theory has been connected with transversal theory.

\begin{definition}(Transversal~\cite{Lai01Matroid})
\label{D:transversal}
Let $\mathbf{F}=\mathbf{F}(J)=\{F_{j}:j\in J\}$ be a family of subsets of $U$.
A transversal of $\mathbf{F}$ is a set $T\subseteq U$ for which there exists a bijection $\pi : T\rightarrow J$ with $t\in F_{\pi(t)}$.
A partial transversal of $\mathbf{F}$ is a transversal of its subfamily.
\end{definition}

To illustrate a transversal and a partial transversal, the following example is given.

\begin{example}
Let $U=\{a, b, c, d, e, f\}, \mathbf{F}=\mathbf{F}(J)=\{F_{1}, F_{2}, F_{3}\}$ and its index set $J=\{1, 2, 3\}$, where $F_{1}=\{a, b, c\}, F_{2}=\{a, d, e\}, F_{3}=\{b, e, f\}$.
Then $T=\{a, d, f\}$ is a transversal of $\mathbf{F}$.
Suppose $J'=\{1, 3\}\subseteq J=\{1, 2, 3\}$, then $T'=\{b, e\}$ is a transversal of $\mathbf{F}(J')$.
Hence $T'$ is a partial transversal of $\mathbf{F}$.
\end{example}

Through transversal theory, a matroid, called transversal matroid, is induced by a family of some subsets of a universe.

\begin{proposition}(\cite{Lai01Matroid})
\label{P:transversalmatroid}
Let $\mathbf{F}=\mathbf{F}(J)=\{F_{j}:j\in J\}$ be a family of subsets of $U$.
Then $M_{T}(\mathbf{F})=(U, \mathbf{I}_{T}(\mathbf{F}))$ is a matroid where $\mathbf{I}_{T}(\mathbf{F})$ is the family of partial transversals of $\mathbf{F}$.
We say $M_{T}(\mathbf{F})$ is the transversal matroid induced by $\mathbf{F}$.
\end{proposition}

The following proposition gives a necessary and sufficient condition for a covering matroid to be a transversal matroid.

\begin{proposition}
Let $\mathbf{C}=\{K_{1}, \cdots, K_{m}\}$ be a covering of $U$ and $k_{1}, \cdots, k_{m}$ a group of nonnegative integers.
Then, $k_{1}=\cdots=k_{m}=1$ if and only if $M(\mathbf{C}; k_{1}, \cdots, k_{m})$ is a transversal matroid.
\end{proposition}

\begin{proof}
According to Definition~\ref{D:coveringmatroid}, Definition~\ref{D:transversal} and Proposition~\ref{P:transversalmatroid}, it is straightforward.
\end{proof}

If a matroid is a transversal matroid, then is it a covering matroid?
From the definitions of covering matroids and transversal matroids, we can easily obtain the following proposition.

\begin{proposition}
Let $\mathbf{F}=\mathbf{F}(J)=\{F_{j}:j\in J\}$ be a family of subsets of $U$ and $M_{T}(\mathbf{F})$ the transversal matroid induced by $\mathbf{F}$.
Then $M_{T}(\mathbf{F})$ is a covering matroid.
\end{proposition}

\begin{proof}
Suppose $F=U-\cup\mathbf{F}$.
According to Definition~\ref{D:coveringmatroid}, Definition~\ref{D:transversal} and Proposition~\ref{P:transversalmatroid}, there exists a covering $\mathbf{C}=\mathbf{F}\cup\{F\}$ such that $M_{T}(\mathbf{F})=M(\mathbf{C}; k_{1}, \cdots, k_{|J|}, 0)$ where $k_{1}=\cdots=k_{|J|}=1$.
Therefore $M_{T}(\mathbf{F})$ is a covering matroid.
\end{proof}

\subsection{Transversal matroids and partition matroids}
In this subsection, we investigate the relationships between transversal matroids and partition matroids.
According to the definitions of transversal matroids and partition matroids, a partition matroid is a transversal matroid when the given group of nonnegative integers are equal to one.

\begin{proposition}
\label{P:partitionformstransversal}
Let $\mathbf{P}=\{P_{1}, \cdots, P_{m}\}$ be a partition of $U$ and $k_{1}, \cdots, k_{m}$ a group of nonnegative integers.
If $k_{1}=\cdots=k_{m}=1$, then $M(\mathbf{P}; k_{1}, \cdots, k_{m})$ is a transversal matroid.
\end{proposition}

Similarly, in the following, we will study under what conditions a transversal matroid is a partition matroid.

\begin{proposition}
\label{P:transversalformspartition}
Let $\mathbf{F}=\mathbf{F}(J)=\{F_{j}:j\in J\}$ be a family of subsets of $U$ and $M_{T}(\mathbf{F})$ the transversal matroid induced by $\mathbf{F}$.
If $F_{i}\cap F_{j}=\emptyset$ for all $i, j\in J, i\neq j$, then $M_{T}(\mathbf{F})$ is a partition matroid.
\end{proposition}

According to the above two propositions, if a family of subsets of a universe is a partition and the given group of integers are equal to one, then the induced matroid is a partition matroid and a transversal matroid.
In the following, we will investigate some properties of a matroid which is a partition one and a transversal one.
First, we will introduce a type of matroids, namely, 2-circuit matroids.

\begin{definition}(2-circuit matroid~\cite{WangZhuZhuMin12matroidalstructure})
\label{D:2-circuitmatroid}
Let $M=(U, \mathbf{I})$ be a matroid. If for all $C\in\mathbf{C}(M), |C|=2$, then we say $M$ is called a 2-circuit matroid.
\end{definition}

In~\cite{WangZhuZhuMin12matroidalstructure}, Wang et al. obtain a matroid through an equivalence relation on a universe and investigate some properties of the matroid.
Since there is a one-to-one between an equivalence relation and a partition, we introduce a proposition from the viewpoint of partitions.

\begin{proposition}(\cite{WangZhuZhuMin12matroidalstructure})
\label{P:M(P)isa2-circuit}
Let $\mathbf{P}$ be a partition on $U$ and $M(\mathbf{P})=(U, \mathbf{I}(\mathbf{P}))$ the matroid induced by $\mathbf{P}$, where $\mathbf{I}(\mathbf{P})=\{X\subseteq U:\forall P\in\mathbf{P}, |X\cap P|\leq 1\}$.
Then $M(\mathbf{P})$ is a 2-circuit matroid.
\end{proposition}

When a matroid is a transversal matroid and a partition one, the matroid is actually a 2-circuit matroid.

\begin{proposition}
Let $\mathbf{P}=\{P_{1}, \cdots, P_{m}\}$ be a partition of $U$ and $k_{1}, \cdots, k_{m}$ a group of nonnegative integers.
If $M(\mathbf{P}; k_{1}, \cdots, k_{m})$ is a partition matroid and a transversal matroid, then $M(\mathbf{P}; k_{1}, \cdots, k_{m})$ is a 2-circuit matroid.
\end{proposition}

\begin{proof}
Since $M(\mathbf{P}; k_{1}, \cdots, k_{m})$ is a transversal matroid, according to Definition~\ref{D:transversal} and Proposition~\ref{P:transversalmatroid}, we obtain $k_{1}=\cdots=k_{m}=1$.
According to Definition~\ref{D:partitionmatroid}, $\mathbf{I}(\mathbf{P}; k_{1}, \cdots, k_{m})=\{X\subseteq U:|X\cap P_{i}|\leq 1\}$.
According to Proposition~\ref{P:M(P)isa2-circuit}, we see $M(\mathbf{P}; k_{1}, \cdots, k_{m})$ is a 2-circuit matroid.
\end{proof}

\subsection{2-circuit matroids and partition-circuit matroids}
Wang et al.~\cite{WangZhuZhuMin12matroidalstructure} have established a one-to-one correspondence between equivalence relations and 2-circuit matroids.
Since there is a one-to-one correspondence between equivalence relations and partitions, we see partitions and 2-circuit matroids are one-to-one corresponding.
The following proposition is obtained.

\begin{proposition}
Let $M$ be a 2-circuit matroid.
Then $M$ is a partition matroid.
\end{proposition}

\begin{proof}
According to Proposition~\ref{P:M(P)isa2-circuit}, it is straightforward.
\end{proof}

In~\cite{LiuZhuZhang12Relationshipbetween}, we study the properties of partition matroids through $k$-rank matroids and give an expression of the dual matroid of a partition matroid.

\begin{proposition}(\cite{LiuZhuZhang12Relationshipbetween})
Let $\mathbf{P}=\{P_{1}, \cdots, P_{m}\}$ be a partition on $U$ and $k_{1}, \cdots, k_{m}$ a group of nonnegative integers.
Then,\\
\centerline{$M^{*}(\mathbf{P}; k_{1}, \cdots, k_{m})=M(\mathbf{P}; |P_{1}|-r_{M(P_{1}, k_{1})}(U), \cdots, |P_{m}|-r_{M(P_{m}, k_{m})}(U))$,}
where $r_{M(P_{i}, k_{i})}(X)=min\{|X\cap P_{i}|, k_{i}\}$ for all $X\subseteq U$.
\end{proposition}

\begin{corollary}
\label{C:MandthedualofM}
Let $\mathbf{P}=\{P_{1}, \cdots, P_{m}\}$ be a partition on $U$ and $k_{1}, \cdots, k_{m}$ a group of nonnegative integers.
Then $M^{*}(\mathbf{P}; 1, \cdots, 1)=M(\mathbf{P}; |P_{1}|-1, \cdots, |P_{m}|-1)$, i.e., $\mathbf{I}^{*}(\mathbf{P}; 1, \cdots, 1)=\{X\subseteq U:|X\cap P_{i}|\leq |P_{i}|-1\}$.
\end{corollary}

In~\cite{LiuZhu12characteristicofpartition-circuitmatroid}, since a partition satisfies the circuit axioms of matroids, we introduce a new type of matroids, namely partition-circuit matroids.

\begin{proposition}(\cite{LiuZhu12characteristicofpartition-circuitmatroid})
Let $\mathbf{P}$ be a partition on $U$ and $M_{\mathbf{P}}=(U, \mathbf{I}_{\mathbf{P}})$ the partition-circuit matroid induced by $\mathbf{P}$.
Then, $\mathbf{I}_{\mathbf{P}}=\{X\subseteq U:\forall P\in\mathbf{P}, |X\cap P|\leq |P|-1\}$.
\end{proposition}

For a partition on a universe, we see that the dual matroid of the 2-circuit matroid is the partition-circuit matroid induced by the partition.

\begin{proposition}
Suppose $M$ is a partition-circuit matroid, then $M$ is a partition matroid.
\end{proposition}

\subsection{Double-circuit matroids}
In this section, we study some properties of a matroid which is both a 2-circuit matroid and a partition matroid.
First, a new type of matroids, called double-circuit matroids, are proposed.

\begin{definition}(Double-circuit matroid)
\label{D:double-circuitmatroid}
Let $M=(U, \mathbf{I})$ be a matroid.
$M$ is called a double-circuit matroid if $|C|=2, |C^{*}|=2$ for all $C\in\mathbf{C}(M), C^{*}\in\mathbf{C}^{*}(M)$.
\end{definition}

When a matroid is a 2-circuit matroid and a partition-circuit matroid, we can obtain the following proposition.

\begin{proposition}
\label{P:theconditionofMisdouble-circuitmatroid}
Let $\mathbf{P}=\{P_{1}, \cdots, P_{m}\}$ be a partition on $U$ and $k_{1}=\cdots=k_{m}=1$.
The following conditions are equivalent:\\
(1) $|P_{1}|=\cdots=|P_{m}|=2$,\\
(2) $M(\mathbf{P}; k_{1}, \cdots, k_{m})$ is a double-circuit matroid.
\end{proposition}

\begin{proof}
$(1)\Rightarrow (2)$: Since $k_{1}=\cdots=k_{m}=1$ and $|P_{1}|=\cdots=|P_{m}|=2$, according to Corollary~\ref{C:MandthedualofM}, we obtain $M(\mathbf{P}; 1, \cdots, 1)=M^{*}(\mathbf{P}; 1, \cdots, 1)$, i.e., $\mathbf{I}(\mathbf{P}; 1, \cdots, 1)=\mathbf{I}^{*}(\mathbf{P}; 1, \cdots, 1)=\{X\subseteq U:|X\cap P_{i}|\leq 1\}$.
According to Proposition~\ref{P:M(P)isa2-circuit} and Definition~\ref{D:2-circuitmatroid}, $|C|=2, |C^{*}|=2$ for all $C\in\mathbf{C}(M(\mathbf{P}; 1, \cdots, 1)), C^{*}\in\mathbf{C}^{*}(M(\mathbf{P}; 1, \cdots, 1))$.
According to Definition~\ref{D:double-circuitmatroid}, $M(\mathbf{P}; k_{1}, \cdots, k_{m})$ is a double-circuit matroid.\\
$(2)\Rightarrow (1)$: According to Corollary~\ref{C:MandthedualofM}, $M^{*}(\mathbf{P}; 1, \cdots, 1)=M(\mathbf{P}; |P_{1}|-1, \cdots, |P_{m}|-1)$.
Since $M(\mathbf{P}; k_{1}, \cdots, k_{m})$ is a double-circuit matroid, then $|C|=2, |C^{*}|=2$ for all $C\in\mathbf{C}(M(\mathbf{P}; 1, \cdots, 1)), C^{*}\in\mathbf{C}^{*}(M(\mathbf{P}; |P_{1}|-1, \cdots, |P_{m}|-1))$.
According to Porposition~\ref{P:M(P)isa2-circuit}, $|P_{i}|-1=1$, i.e., $|P_{i}|=2$, for $i=1, \cdots, m$.
\end{proof}

In the following, we investigate some properties of a double-circuit matroid.
First, we introduce the isomorphism in matroid theory.

\begin{definition}(Isomorphism~\cite{Lai01Matroid})
Let $M_{1}=(U_{1}, \mathbf{I}_{1}), M_{2}=(U_{2}, \mathbf{I}_{2})$ be two matroids.
$M_{1}$ and $M_{2}$ are isomorphic and denoted as $M_{1}\cong M_{2}$ if there exists a bijection $\psi: U_{1}\rightarrow U_{2}$ such that for all $X\subseteq U_{1}, X\in\mathbf{I}_{1}$ if and only if $\psi(X)\in\mathbf{I}_{2}$.
\end{definition}

For a matroid, it is a self-dual matroid if the matroid and its dual matroid are isomorphic.
When the matroid and its dual matroid are equal to each other, we say the matroid is an identically self-dual matroid.

\begin{proposition}
Let $M$ be a double-circuit matroid.
Then $M$ is an identically self-dual matroid.
\end{proposition}

\begin{proof}
According to Proposition~\ref{P:theconditionofMisdouble-circuitmatroid} and Corollary~\ref{C:MandthedualofM}, $M=M^{*}$.
Therefore, $M$ is an identically self-dual matroid.
\end{proof}

In summary, we can use the following graph to represent the work of this section.

\begin{figure*}[tb]
    \begin{center}
    \includegraphics[width=4.5in]{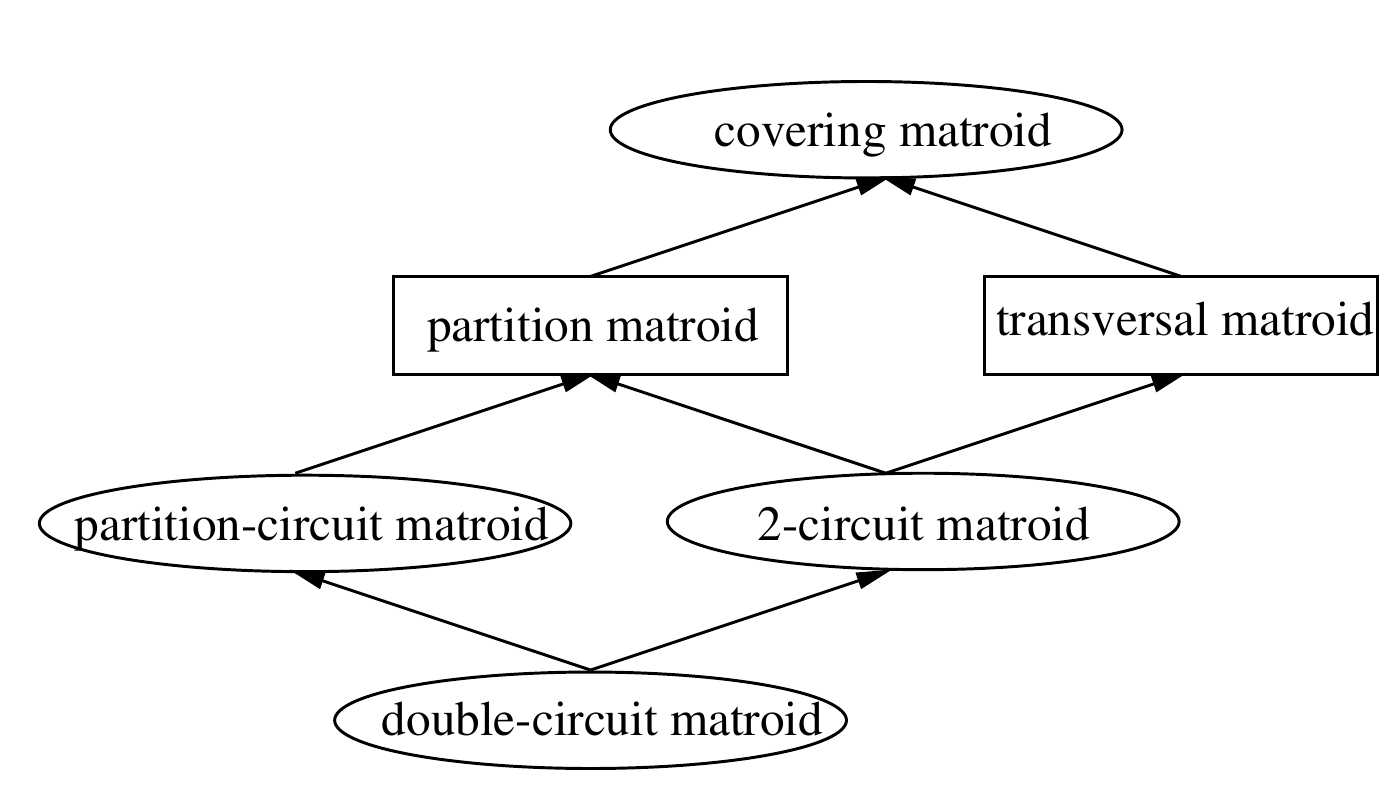}
    \caption{Covering matroids and some special matroids}
    \label{processGraph}
    \end{center}
\end{figure*}

%%%%%%%%%%%%%%%%%%%%%%%%%%%%%%%%%%%%%%%%%%%%%%%%%%%%%%%%%%%%%%%%%%%%%%%%%%%%%%%
%%%%%%%%%%%%%%%%%%%%%%%%%%%%%%%%%%%%%%%%%%%%%%%%%%%%%%%%%%%%%%%%%%%%%%%%%%%%%%%

%%%%%%%%%%%%%%%%%%%%%%%%%%%%%%%%%%%%%%%%%%%%%%%%%%%%%%%%%%%%%%%%%%%%%%%%%%%%%%%
%%%%%%%%%%%%%%%%%%%%%%%%%%%%%%%%%%%%%%%%%%%%%%%%%%%%%%%%%%%%%%%%%%%%%%%%%%%%%%%
\section{Conclusions}
\label{S:conclusions}
In this paper, we proposed a new type of matroids called covering matroids and studied the connections with the second type of covering-based rough sets and some special matroids.
First, through extending a partition to a covering, we used the union of $k$-rank matroids to obtain a new matroid called covering matroid and proved this new matroid was an extension of the partition matroid.
Second, some concepts of covering-based rough sets were studied by covering matroids, such as neighborhood, the lower and upper approximation operators.
Third, we investigated the relationships between covering matroids and some special matroids such as transversal matroids, partition matroids, 2-circuit matroids and partition-circuit matroids.
In future work, we will study the applications of covering matroids.
Since partition matroids are successfully applied to attribution reduction in information systems, as an extension of partition matroids, covering matroids will be used to solve some problems of information systems.

%%%%%%%%%%%%%%%%%%%%%%%%%%%%%%%%%%%%%%%%%%%%%%%%%%%%%%%%%%%%%%%%%%%%%%%%%%%%%%%
%%%%%%%%%%%%%%%%%%%%%%%%%%%%%%%%%%%%%%%%%%%%%%%%%%%%%%%%%%%%%%%%%%%%%%%%%%%%%%%
\section*{Acknowledgments}
This work is supported in part by the National Natural Science Foundation of China under Grant No. 61170128, the Natural Science Foundation of Fujian Province, China, under Grant Nos. 2011J01374 and 2012J01294, the Science and Technology Key Project of Fujian Province, China, under Grant No. 2012H0043 and State key laboratory of management and control for complex systems open project under Grant No. 20110106.
%%%%%%%%%%%%%%%%%%%%%%%%%%%%%%%%%%%%%%%%%%%%%%%%%%%%%%%%%%%%%%%%%%%%%%%%%%%%%%%
%%%%%%%%%%%%%%%%%%%%%%%%%%%%%%%%%%%%%%%%%%%%%%%%%%%%%%%%%%%%%%%%%%%%%%%%%%%%%%%

%\bibliographystyle{splncs}
%\bibliography{E:/liu/matroid/bib/Roughsets1}

% Set the ending of a LaTeX document
\end{document}